\documentclass{article}

\usepackage{microtype}
\usepackage{graphicx}
\usepackage{subfigure}
\usepackage{booktabs} 

\usepackage{float}
\usepackage[normalem]{ulem}
\usepackage{makecell,multirow,diagbox}  
\usepackage{float}
\usepackage{adjustbox}
\usepackage{lscape}
\usepackage{amsfonts}
\usepackage{amsthm}
\usepackage{amsmath}

\usepackage{amssymb}
\usepackage{color}
\usepackage{xcolor}
\usepackage{algorithm}
\usepackage[noend]{algorithmic}
\usepackage{enumitem}
\usepackage{color, colortbl}\usepackage{url}
\definecolor{gray}{HTML}{545454}
\usepackage{wrapfig}
\usepackage{booktabs}
\usepackage{multirow}

\usepackage{hyperref}

\usepackage[accepted]{icml2024}

\usepackage[capitalize,noabbrev]{cleveref}

\usepackage[textsize=tiny]{todonotes}

\icmltitlerunning{An Auction-based Marketplace for Model Trading in Federated Learning}

\usepackage{amsmath}
\usepackage{amsfonts}

\newcommand{\ie}{\emph{i.e., }}
\newcommand{\eg}{\emph{e.g., }}

\newcommand{\st}{\emph{s.t. }}

\newcommand{\wrt}{\emph{w.r.t. }}

\renewcommand{\algorithmicensure}{}

\newcommand{\real}{\mathbb R}

\newcommand{\mcM}{\mathcal{M}}

\newcommand{\mba}{\mathbf{a}}
\newcommand{\mbb}{\mathbf{b}}

\newcommand{\mbp}{\mathbf{p}}

\newcommand{\mbs}{\mathbf{s}}

\newcommand{\mbA}{\mathbf{A}}

\newtheorem{theorem}{Theorem}[section]
\newtheorem{definition}{Definition}[section]
\newtheorem{assumption}{Assumption}[section]

\newtheorem*{theorem*}{Theorem}
\newtheorem*{definition*}{Definition}
\newtheorem*{assumption*}{Assumption}
\newtheorem*{conjecture*}{Conjecture}
\newtheorem*{claim*}{Claim}
\newtheorem*{lemma*}{Lemma}
\newtheorem*{proposition*}{Proposition}
\newtheorem*{property*}{Property}
\newtheorem*{fact*}{Fact}
\newtheorem*{corollary*}{Corollary}
\newtheorem*{example*}{Example}
\newtheorem*{remark*}{Remark}
\newtheorem*{exercise*}{Exercise}

\begin{document}

\twocolumn[
\icmltitle{An Auction-based Marketplace \\ for Model Trading in Federated Learning}




\begin{icmlauthorlist}
\icmlauthor{Yue Cui}{hkust}
\icmlauthor{Liuyi Yao}{alibaba1}
\icmlauthor{Yaliang Li}{alibaba2}
\icmlauthor{Ziqian Chen}{alibaba1}
\icmlauthor{Bolin Ding}{alibaba2}
\icmlauthor{Xiaofang Zhou}{hkust}

\end{icmlauthorlist}

\icmlaffiliation{hkust}{The Hong Kong University of Science and Technology, Hong Kong SAR}
\icmlaffiliation{alibaba1}{Alibaba Group, Hangzhou, China}
\icmlaffiliation{alibaba2}{Alibaba Group, Seattle, USA}

\icmlcorrespondingauthor{Yue Cui}{ycuias@cse.ust.hk}
\icmlcorrespondingauthor{Liuyi Yao}{yly287738@alibaba-inc.com}
\icmlcorrespondingauthor{Yaliang Li}{yaliang.li@alibaba-inc.com}
\icmlcorrespondingauthor{Ziqian Chen}{eric.czq@alibaba-inc.com}
\icmlcorrespondingauthor{Bolin Ding}{bolin.ding@alibaba-inc.com}
\icmlcorrespondingauthor{Xiaofang Zhou}{zxf@cse.ust.hk}
\icmlkeywords{Machine Learning, ICML}

\vskip 0.3in
]



\printAffiliationsAndNotice 

\begin{abstract}
Federated learning (FL) is increasingly recognized for its efficacy in training models using locally distributed data. However, the proper valuation of shared data in this collaborative process remains insufficiently addressed. In this work, we frame FL as a marketplace of models, where clients act as both buyers and sellers, engaging in model trading. This FL market allows clients to gain monetary reward by selling their own models and improve local model performance through the purchase of others' models. We propose an auction-based solution to ensure proper pricing based on performance gain. Incentive mechanisms are designed to encourage clients to truthfully reveal their model valuations. Furthermore, we introduce a reinforcement learning (RL) framework for marketing operations, aiming to achieve maximum trading volumes under the dynamic and evolving market status. Experimental results on four datasets demonstrate that the proposed FL market can achieve high trading revenue and fair downstream task accuracy.
\end{abstract}


\section{Introduction}

Federated learning (FL) has gained significant attention in recent years as it enables collaborative training of machine learning models from distributed data sources without sharing private information \citep{mcmahan2017communication,konevcny2016federated,yang2019federated,yang2019federated_app,hard2018federated,xu2021federated, leroy2019federated}. FL can be categorized based on participant types into cross-device and cross-silo settings \cite{yang2019federated}. In this study, we specifically focus on cross-silo FL, where participants typically include institutions with moderate to large datasets (\eg hospitals, banks), and the FL model(s) are trained for local use by these participants. 

FL can be conceptualized as a process wherein data is shared, exchanged, and reused through communication rounds. The current model sharing scheme in FL is highly egalitarianism: each client altruistically contributes model updates to the central server \cite{mcmahan2017communication,yang2019federated}, which are subsequently aggregated to form global or personalized models accessible to other participating clients \cite{li2022federated,mcmahan2017communication,yang2019federated,li2021ditto,li2021fedbn}. This can lead to discouraging outcomes to certain clients. For example, a client who already has high-quality local data and thus commendable local model performance may experience marginal performance improvements but enhances the model performance of other clients, potentially including its competitors, thereby diminishing its own market share.

\begin{figure*}[t] 
\vspace{-0.15cm}
\centering
\includegraphics[width=0.95\linewidth]{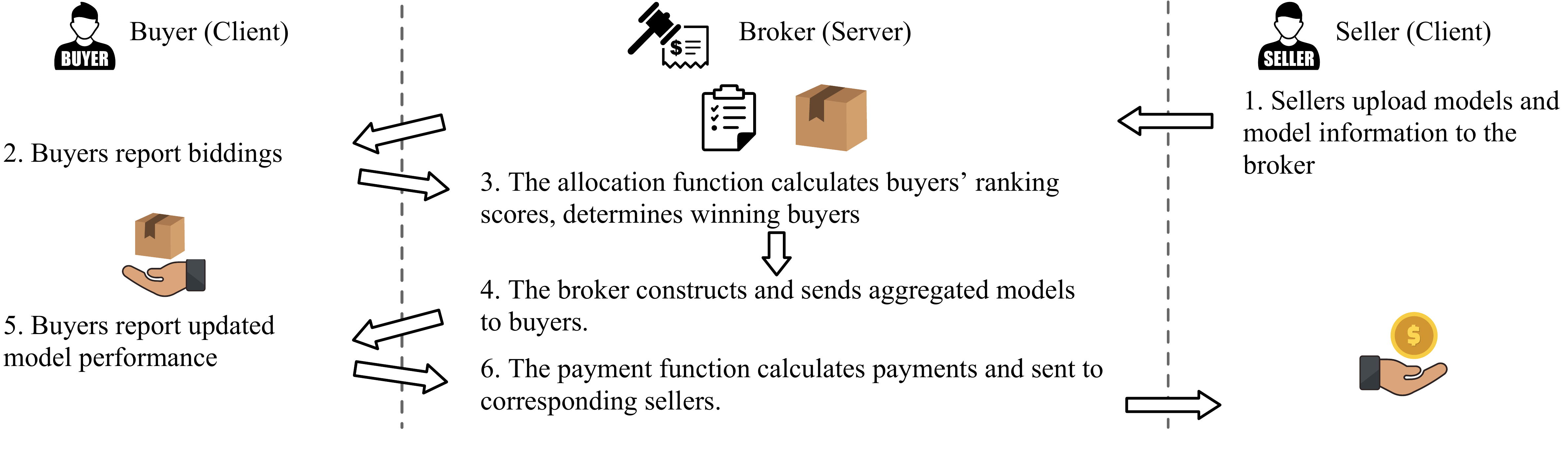}
\vspace{-0.35cm}
\caption{Pipeline of the auction-based FL market.}
\vspace{-0.35cm}
\label{fig:pip_auc}
\end{figure*}

To mitigate this issue, we suggest integrating model valuation into FL. Under this framework, client model updates are no longer freely shared; instead, a client is required to offer monetary compensation to the provider of the update it intends to use. Clients take on dual roles as both model sellers and buyers: as sellers, they offer their local models in exchange for monetary rewards, while as buyers, they seek to acquire models from other clients to enhance their own through model aggregation. This framework effectively establishes an FL model market and each FL communication round is driven by a round of model trading.

Designing an FL model marketplace is challenging. Firstly, given the competitive nature among buyers, a mechanism that ensures buyers with higher willingness to pay are more likely to acquire the goods needs to be devised. Secondly, determining the appropriate pricing for model products is complex for sellers, as the utility of a model depends on the production environment of its user—the model buyers in the FL market. For example, 1\% performance increase for company A leads to revenue increase of \$1M but can only produces \$1,000 for company B. Thirdly, from the perspective of the market platform, maintaining high trading volumes is essential to signal market sentiment and attract more investments. This, however, could be challenging due to the dynamic and evolving market status during FL communication rounds.

To address the first two challenges, auction mechanisms, which have been proven to be successful in the online advertising \citep{zhang2021optimizing,liu2021neural}, offer a potential solution. Ad systems share similarities with the FL environment: advertisers vie for ad slots on the platform, and the pricing of an ad slot is based on user response after display. In the context of FL, an auction mechanism could function by allowing model sellers to offer their models for bidding, with the bidding price designed as a function of future performance gain. Specifically, we propose the design of bidding and payment rules for the FL market, where bidding values represent the highest price a client is willing to pay for a unit of performance increase. Bidding results are determined by the allocation function, and successful bidders aggregate the models they win to build the delivery. The proposed auction pipeline is shown in Figure~\ref{fig:pip_auc}.


To tackle the third challenge, we establish maximizing trading volumes as the primary objective of the marketplace. We propose two strategies, each supported by theoretical proofs, to advance this objective: 1) limiting the number of model copies a client can sell to be smaller than the total number of buyer clients, and 2) in each auction round, having the server randomly authorize a subset of sellers for sale. This design incentivise the true valuation revealing of buyers, which is essential for an effective auction mechanism \cite{krishna2009auction,li2020multi}. Subsequently, we introduce a reinforcement learning (RL)-based solver to optimize the allocation function mentioned earlier, aiming to achieve maximum trading volume while accommodating the dynamic and evolving nature of the market status.

Our contributions can be summarized as follows.

$\bullet$  We revisit model exchanging in FL from the perspective of economic trading. We treat each client's local model as an asset. A client should be paid with monetary rewards if its model is used for FL aggregation. 

$\bullet$ We introduce a market design based on auction, with performance gain based pricing for bidding and payment. Two theoretical principles are proposed to encourage true valuation revealing of buyer clients. 

$\bullet$ An RL based marketing operation solver is proposed to maximize the revenue generated by the marketplace with certain properties satisfied.

$\bullet$ Extensive experiments demonstrate the proposed approach incentivizes high trading volume with compatible model accuracy compared with baselines.


\section{Preliminary}
\label{sec:pre}


\textbf{Ingredients of the Auction-based FL Market.} The auction based FL market contains three key roles: model buyer, model seller, and broker.

Each \textit{client} in the federation can act as a \textbf{model seller}. A seller sells its model parameters or gradients, denoted as $A_i^t$, to generate income. $i$ denotes client $i$ and $t$ denotes the current round of auction. 

Each \textit{client} in the federation can also act as a \textbf{model buyer}. A buyer obtains aggregated model by paying monetary rewards to the corresponding model sellers via bidding. Each model buyer announces a bidding vector $\mbb_i^t$ to express its willingness to buy models on the market. Successfully bidder will receive model asset based on FL aggregation (see Section \ref{sec:goods}). Each buyer operates under a budget constraint (defined in Assumption \ref{ass:ir}).

The \textit{central server} serves as the \textbf{broker}. The broker receives bidding values, bidding-related information announced by model buyers, and goods-related information sent by model sellers. It is responsible for calculating the auction result and payment, as well as performing model aggregation.

The broker plays a crucial role in FL market. It prevents the risk of information leakage across clients, ensuring that a client's model/gradient is not directly accessible to other clients. It also determines an optimal mechanism towards the data markets. 

Based on the above clarifications, the problem studied in this paper can be described as follows.

\textbf{Problem Statement}. The design of an auction game usually includes characterizing of the \textit{allocation function} and \textit{payment function} \cite{krishna2009auction}. In this paper, we take the view of a trading platform, \ie on the server side, to design the allocation and payment functions that attracts as high trading volume as possible.

The definitions of allocation and payment functions are as follows. And the definition of trading volumes is deferred to section \ref{eq:obj}
.
\begin{definition}\textbf{Allocation Function}:
Denoted as $f_a$, an allocation function $f_a: (\mbb^t,D^t)\rightarrow \real^{N\times N}$, takes biding $\mbb^t$ and other possible information $D^t$, \eg the model parameters of clients $\mbA^t$, at round $t$ as input, outputs the ranking scores of bidders and determines the auction results. Denote the ranking score as $\mba^t$. 
\end{definition}

\begin{definition}\textbf{Payment Function}:
Denoted as $f_p$, the payment function $f_p: (\mbb^t, \mba^t)\rightarrow \real^{N\times N}$ outputs the price each bidder should pay. Denote the price as $\mbp^t$. 
\end{definition}
Note that although not necessarily all clients participate as model sellers or buyers, we define $\mba^t \in \real^{N\times N}$ and $\mba^t \in \real^{N\times N}$, where the corresponding values can be masked or set as zeros if not available in a round. For instance, $p^t_{i,j}=0$ if the bidding from client $i$ to client $j$'s model is unsuccessful. The allocation and payment functions are deployed and maintained on the server side. Key notations used in this paper are summarized in Appendix \ref{app:notation}.


\label{sec:app:notation}



\section{Methodology}
\label{sec:method}
In this section, we introduce details of the proposed FL market. We begin by establishing a connection between performance gain and monetary revenue, formulating how to price a model based on the performance improvement it provides to buyer clients.
\subsection{Market Design}
\label{sec:design}
\textbf{Performance Gain Based Pricing.}
\label{sec:performance_revenue}
A primary motivation for a client to participate in FL is to enhance its model performance. Performance gain, defined as the relative improvement in a client's performance before and after participating in FL, has been demonstrated to be an effective metric for quantifying this benefit \citep{yao2022federated}. The performance gain can be formulated as follows.

\begin{definition}\textbf{Performance Gain}:
Denote the initial performance of a model on client $i$ as $\mcM^0_i$, which could be the performance of isolated training \citep{yao2022federated} setting, and the performance after $t$ rounds of auction as $\mcM^t_i$. The performance gain is calculated as:
\begin{equation}
G_i^t=(\mcM^t_i-\mcM^0_i)/\mcM^0_i,
\label{eq:gain}
\end{equation}
where $G_i^0=0$ and we assume $\mcM^0_i$ is non-zero. Further, we can define the difference in performance gain, \ie 
\begin{equation}
\Delta G_i^t=G_i^t-G_i^{t-1},
\end{equation}
where $ t\geq 1$.
\end{definition}
As $\Delta G$ is more frequently used in the paper, we refer $\Delta G$ as the performance gain for short.

The bidding price connects buyers' willingness to pay and the value of the product. In the context of the FL market, this implies that a buyer should increase its bidding price to achieve a higher performance gain. Similarly, a seller providing more performance gain to a buyer should receive more monetary rewards. To establish this connection, we propose using unit performance gain-based bidding. As multiple sellers participate in a round of auction and a buyer can obtain products from more than one seller, the bidding price is defined as a multi-dimensional vector:

\begin{definition} \textbf{Bidding Price}:
In round $t$, model buyer $i$ reports $\mbb_i^t=[b_{i,1}^t,...,b_{i,N}^t]$, where $b_{i,j}^t$ denotes the bidding value for seller $j$'s model, representing the highest price that the buyer is willing to pay for a unit of performance increase if it could get the model. 
\end{definition}


\begin{definition} \textbf{Payment Rule}:
If buyer $i$ successfully bids seller $j$'s model at round t, the payment function calculates payment $p_{i,j}^t$, then buyer $i$ pays $p_{i,j}^t\Delta G_i^t$ to seller $j$.
\end{definition}

The bidding price and payment rule designed above serves as the buyer valuation realization in this paper. It is based on the central idea that the value of a client's model should not be considered in isolation but should be related to the performance gain it provides to the buyer clients.

\textbf{FL Market Assumptions.}
\label{sec:assumption}
With the definitions provided, we make the following assumptions to simplify the problem. 

\begin{assumption} \textbf{Individual Rationality.}
 1) No gain, no pay. A client will not pay in a round if it experiences a worse outcome than the previous round of auction, \ie if $\Delta G_i^t\leq0$, client buyer $i$ does not need to pay. 2) Budget constraint. Denote the utility of a unit performance increase of buyer $i$ as $u_i$, each client can pay up to $u_i$ for a unit performance increase, \ie $\sum_{j=1}^N b_{i,j}^t\leq u_i$.
\label{ass:ir}
\end{assumption}

\begin{assumption} \textbf{Unified Bidding Strategy.} All clients follow the same bidding strategy.
\label{ass:uni}
\end{assumption}

Assumption \ref{ass:ir}.1 is based on the idea that a client should not be worse off than if it leaves the federation \citep{cong2020vcg}, while Assumption \ref{ass:ir}.2 is founded on the concept that the utility of performance increase is limited. Note that $u_i$ is maintained by each client itself but not be made public. 
Moreover, the unified bidding strategy in Assumption \ref{ass:uni} is commonly employed in auction game design \citep{krishna2009auction,myerson1981optimal} to simplify the problem setting so as to analyze principles of the designed market.

As security issue is out the scope of this paper, for the setting of FL environment, we follow the commonly used no-adversary setting \cite{yang2019federated,mcmahan2017communication,Tang21cross_silo,li2020federated,li2021fedbn,incentive2022survey}, where the local model weights clients submitted to the central server are not manipulated and there is no collusion among clients.



\textbf{Objectives of the FL Market.}
\label{sec:obj}
The second-price auction is an elegant and classical mechanism in which each bidder will turn to a truthful bidding strategy, which is a dominant strategy if bidders realize their valuation in a single-item auction \cite{myerson1981optimal}. Other mechanisms, such as first-price auction, can result in bid shading where bidders will bid lower than their actual valuation for the items. In each FL round, buyers cannot know in advance about the value of models they will obtain in subsequent rounds. Instead, their bidding strategies are based on predicted model performance utility increases, which complicates the analysis of client behaviors with bid shading considered. For simplicity, we develop the FL market with the second-price auction concept. Our focus then shifts to addressing FL-specific challenges, such as the dynamic system and the unknown valuations. Based on this, we can now formulate the objective of FL market. 

Given a set of clients $\{c_i,...,c_N\}$, a bidding strategy, for a $T$ rounds auction, an FL market aims to achieve \textbf{trading volume maximization}. More specifically, we seek to maximize the cumulative revenue of model trading. Similar to traditional marketplace where the trading volume is defined as the number of shares traded multiplies price per share, in FL market, the objective of the market can be mathematically calculated as
\begin{equation}
\begin{aligned}
\max R&=\max \sum_{t=0}^T \sum_{i=1}^{N} (\sum_{j=1}^{N} x_{j,i}^t p_{j,i}^t \Delta G_j^t);\\
\st  & \text{social fairness constraint},\\
 &\text{efficient allocation constraint},
\end{aligned} 
\label{eq:obj}
\end{equation}
where the indicator of winning state is represented by $x_{i,j}^t$, such that $x_{i,j}^t=1$ if buyer $i$ successfully bids on seller $j$'s model at round $t$; otherwise, it is 0.

When maximizing the total value generated by the marketplace, we also aim to achieve certain desired properties: the social fairness constraint and efficient allocation constraint. For the \textbf{social fairness constraint}, we define it as the revenue gained by a client should not be lower than a portion of a benchmark revenue, \ie
	\begin{equation}  
	R^t_i\geq \eta R_0\text{, for $\forall$ i},
	\end{equation}
where $R^t_i=\sum_{j=1}^{N} x_{j,i}^t p_{j,i}^t \Delta G_j^t$, $\eta$ and $R_0$ are manually set hyper-parameters. The other constraint, \textbf{efficient allocation constraint}, aims at assigning the auctioned items to the bidders who value them the most, thereby maximizing the total value from the allocation. Inspired by \cite{myerson1981optimal}, we primarily consider two sub-constraints: 1) monotone allocation, wherein the winning bidder would still win the auction if it reports a higher bid, and 2) critical bid-based pricing, where the pricing rule is based on the critical bid, \ie the second-highest bid that would not lose the auction.

\subsection{Principles to Incentives True Valuation Revealing}
\label{sec:principle}
It has been proven that in auction design, creating an environment where bidders are incentivized to reveal their truth valuations (\ie the bidding value equals the buyer's utility on the product) eliminates the need to account for strategic behaviors \citep{zhang2021optimizing}, resulting in reliable and predictable inputs for objective optimization. 


Note that during FL course, the identification of the model buyers and sellers will not change, which means each model buyer will try to predict its monetary utility of a unit performance increase $u_i^t(M_j^{t-1},A_i^{t-1})$ based on the model information $M_j^{t-1}$ from the seller $j$ and its own current model weights $A_i^{t-1}$. In simplicity, we use the term $X_{i,j}^{t-1}$ to represent $(M_j^{t-1},A_i^{t-1})$ and assume $u_i^t(M_j^{t-1},A_i^{t-1})\leq u_i$, where $u_i$ is the objective utility gain for a unit of performance increase as defined in Assumption \ref{ass:ir}. As a consequence, in our case, bidding of buyer $i$ towards seller $j$'s model can be reformulated as $ b_{i,j}^t=u_i^t(X_{i,j}^{t-1})$. 

In the auction-based FL market, each seller can be seen as an auctioneer selling a number of identical objects. However, although digital products such as models have zero marginal costs \citep{pei2020survey}, arbitrarily setting the number of model copies for sale can be problematic. We propose the following theorem.

\begin{theorem} 
When each model seller has no limits on the number of the sold model copies, no model buyer will have the incentive to submit a bid equaling its predicted monetary utility unless there exists a constraint on the number of model copies:
\begin{equation}
\left\{
\begin{aligned}
&b_{i,j}^t < u_i^t(X_{i,j}^{t-1}) , \quad \text{if} \quad k = N;   \\
&b_{i,j}^t \leq u_i^t(X_{i,j}^{t-1}), \quad \text{if}\quad k < N.
\end{aligned}
\right .
\end{equation}

\label{th:copy}
\end{theorem}

The proof of Theorem \ref{th:copy} is deferred to Appendix \ref{app:proof1}. Based on the theorem, denote the number of model copies sold by a seller client as $k$, we set $k<N$. For simplicity, we assume all sellers share the same value of $k$. This means only buyers with top-k allocation score $a_{i,j}^t$ win the auction \wrt seller $j$'s model. However, with $k<N$ applied, a competitive environment is established. Due to the individual rationality of clients, they seek to operate with the optimal bidding strategy. In that case, it can be problematic if all sellers are authorized to sell at each round. To deal with this, we propose the next theorem. 
\begin{theorem}
Suppose the clients authorized as model sellers at each round are determined by a selective mechanism on the server. If all clients are authorized to sell at each round, model buyers will not bid truthfully. But there exists a selection mechanism $S$ that can urge a sufficient number of model buyers to bid truthfully. 
\label{th:curse}
\end{theorem}

As a sketch, Theorem \ref{th:curse} holds because for a model buyer, the allocation is partially observed, it does not have the expectation that a target model will be available in future rounds and thus will efficiently participate in each round of auction. Detailed proof is deferred to Appendix \ref{app:proof2}. Therefore, together with Theorem \ref{th:copy}, we adopt a simple yet effective mechanism to incentivize truthful bidding: at an auction round, only $k$ $(k<N)$ copies of a model are allowed to be sold and only a random subset of $m$ $(m<N)$ sellers' models are authorized for selling. This ensures true valution revealing of buyers.

\subsection{Reinforcement Learning Based Auction Allocation}
\label{sec:solver}
In auction settings, the term `allocation' denotes the assignment of items to bidders based on auction results. This process involves establishing who receives which goods at what price. Addressing the challenge of fulfilling specific allocation objectives necessitates an optimized decision-making strategy. To this end, we propose a RL based framework for effective allocation. The RL algorithm is tasked with allocating models to purchasers in accordance with predetermined objectives and constraints, as detailed in Equation \ref{eq:obj}. Given the inherent complexities within round-based auctions in FL scenarios—where model availability and buyer valuation are dynamic—the adaptive capabilities of RL are particularly pertinent. RL's suitability stems from its design for environments where agents must make sequential decisions that influence subsequent choices.

In RL, the agent iteratively interacts with an evolving environment. It observes the current environmental state and performs actions that prompt state transitions, thereby receiving feedback in the form of rewards. The agent's overarching goal is to develop a policy that maximizes aggregate rewards over time, synonymous with optimizing the value function. Within the FL marketplace, the platform acts as the RL agent, with auction allocations constituting its actions, and the current market conditions forming the state. Subsequent sections will delineate the specific definitions of the state, actions, environment updates, rewards, and value within this framework.

$ \bullet$  State. State of an FL market includes the reported information, \ie biddings of buyers, the model performance of sellers, and model weights/gradients of sellers.

$ \bullet$ Action. Action is defined as the output of allocation function, $\{\mba^t_1,...,\mba^t_{N}\}$, where $\mba^t_i$ is the ranking score of buyers \wrt the bidding successful probability towards seller $i$'s model.

$ \bullet$ Environment Update. After determining the actions, for each seller, the broker marks top-k buyers as winners for each buyer. For each buyer, the broker aggregates successful bids, and then sends the aggregated model to the buyer. The buyer performs evaluation and local updates based on the obtained model and reports corresponding information to the broker. 

$ \bullet$ Reward. Once actions are taken and the environment update is observed, we calculate the reward across all clients, denoted as $r^t$, as follows: 
\begin{equation}
r^t=\sum_{i=1}^N R_i^t-\max(0, \eta R_0-R_i^t).
\label{eq:rs}	
\end{equation}
The first term is the objective in Equation~\ref{eq:obj}. The second term adapts reward shaping mechanism \citep{zhang2021optimizing,karimireddy2022mechanisms} to achieve the social fairness goal, suggesting that if an individual's income is lower than a threshold $\eta R_0$, a penalty will be made.
 
$ \bullet$ Value. Value is calculated as cumulative future reward, \ie $y_t=\sum_{\tau=t}^T r_\tau.$


 

 

\textbf{Auction Result Calculation.}
We adopt the policy gradient based deep reinforcement learning approach advance actor critic (A2C) \citep{mnih2016asynchronous}. The A2C framework contains two networks: the policy net takes the state as input and outputs the action; the value net takes state as input and gives an estimation of the value of the action. Since the revenue made in an FL market accumulates as the FL course continues, we expect that an action taken will result in better future rewards. Temporal-difference (TD) learning is a core learning technique in reinforcement learning \citep{sutton2018reinforcement,sutton1988learning,kaelbling1996reinforcement}, which learns good estimates of the expected return by bootstrapping from other
expected-return estimates \citep{van2016true}. We here adopt TD learning to train the policy net and value net. 

Instead of considering the action space from buyers' perspectives, which is exponential to the number of sellers, we propose to define the action space by sellers. Denote that for an auction round, the total number of authorized sellers is $m$. We formulate the state at round $t$ as $\mbs^t=[\mbs_{m_1}^t,...,\mbs_{m_m}^t]$, where $\mbs_i^t=(\mathbf{\mcM}^t_{i}, \mathbf{A}^t_{i},\mathbf{\mcM}^t_{-i}, \mathbf{A}^t_{-i})$, where $\mathbf{\mcM}^t_{-i} \in \real^{(N-1) \times 1}$, $\mathbf{A}^t_{-i} \in \real^{(N-1) \times d_{model}}$, and $d_{model}$ denotes the size of the model, which represents what seller $i$ is offering and who are the buyers. The policy net takes the normalized $\mbs^t$ as input and calculates the probability weights of a seller choosing each buyer. Denote the policy net as $\pi(\mbs^t;\theta )$ and ranking scores as $\mba^t$, we have:
\begin{equation}
\mba^t=\mbb^t\circ\pi(\mbs^t;\theta),
\label{eq:mono}
\end{equation}
where $\circ$ denotes the Hadamard product. Since $\pi(\mbs^t;\theta)$ is the probability weights of a softmax layer ($>0$), Equation~\ref{eq:mono} is monotonic to $\mbb^t$. Thus, the above formulation satisfies the monotone allocation objective defined in Section~\ref{sec:obj}. For each seller, the top-k-ranked buyers are marked as winners. 

\textbf{Model Aggregation and Goods Delivery.}
\label{sec:goods}
The successful bidders obtain an aggregated model as the delivered good of the current round. Given a bidder $i$, assume that it has been marked as successful on a set of sellers' products, denoted as $\{\phi_i\}$. $\{\phi_i\}$ is calculated by the output of the policy net on the server. The server then applies the aggregation function of FL algorithms across the model weights/gradients of client $i$ and $\{\phi_i\}$. In our setting, for simplicity, we adopt FedAvg \cite{mcmahan2017communication} such as the FL algorithm, which is a classic FL approach that has been widely applied due to its simplicity and low communication cost. Denote the current model weights of bidder $i$ is $A^t_i$ and the goods it receives as $A^{t+1}_i$. $A^{t+1}_i$ is calculated as 
\begin{equation}
A^{t+1}_i=\sum_{j\in \phi_i\cup \{i\}}\frac{n_j}{n}A^{t}_j,
\end{equation}
where $n_j$ is the number of samples of client $j$'s local data and $n$ is the total number of sample of all clients in $\phi_i\cup \{i\}$.

As FedAvg serves as the basic paradigm of most state-of-the-art horizontal FL approaches \cite{li2021fedbn,li2020federated,praneeth2019scaffold}. The generalization of the above equation to other FL frameworks is straightforward. The obtained $A^{t+1}_i$ is then used for the local update of client $i$. Together with $A^{t+1}_i$, the obtained performance $\Delta M_i^{t+1}$ performance gain $\Delta G_i^{t+1}$ are utilized to form the state of the FL market, \ie $\mbs^{t+1}$, in the next round.

\textbf{Payment Calculation.}
Recall that for each seller, the top-k-ranked buyers are marked as winners. To approximate the payment of the winning buyers and meet the desired property of critical bid-based pricing, we propose the following payment function based on \citep{zhang2021optimizing}:
\begin{equation}
\mbp^t_{k_i}=\Delta G_{k_i}^t\mbb_{k_{i+1}}^t\frac{\mba_{k_{i+1}}^t}{\mba_{k_i}^t},
\label{eq:rl_payment}
\end{equation}
where $k_i$ is the i-ranked client of the top-k clients and $\Delta G_{k_i}^t$ is the improved performance after environment update. 

\textbf{Update the RL Solver.}
The value net respectively takes state at round $t$ and the updated state at $t+1$ as input, \ie $\mbs^t$ and $\mbs^{t+1}$ as input, and make prediction on the return:
\begin{equation}
\begin{aligned}
&\hat{v}_t=v(\mbs^t;\omega), \hat{v}_{t+1}=v(\mbs^{t+1};\omega).
\label{eq:value}
\end{aligned}
\end{equation}

We calculate the TD target $\hat{y}_t$ as the summation of true ad-hoc reward at round $t$ and the estimated return at round $t+1$. And the TD error $\hat{\delta}_{t}$ is calculated as the difference between the estimated return at round $t$ and the TD target. Mathematically: 
\begin{equation}
\begin{aligned}
\hat{y}_t=r_t+\hat{v}_{t+1},
\hat{\delta}_{t}=\hat{v}_t-\hat{y}_t.
\end{aligned}
\label{eq:td}
\end{equation}
Based on Equation~\ref{eq:td}, we update the policy net and the value net as follows, where $\alpha$ and $\beta$ are the learning rates.
\begin{equation}
\begin{aligned}
\theta' &\leftarrow \theta -\beta \hat{\delta}_{t} \nabla_{\theta}ln\pi(\mbs^t;\theta),\\
\omega ' &\leftarrow \omega -\alpha \hat{\delta}_{t} \nabla_{\omega}v(\mbs^t;\omega).
\end{aligned}
\label{eq:update}
\end{equation}

Pseudo code of the training of the RL-based solver is deferred to Appendix \ref{app:rl_solver}. And a discussion regarding privacy issues of the FL market can be found in Appendix \ref{app:privacy}.


\section{Experiment}
\label{sec:exp}
\subsection{Experiment Setup}

\textbf{Datasets.}
\label{sec:app:setup}
We evaluate on four datasets: MNIST \cite{lecun1998gradient}, FMNIST \cite{xiao2017fashion}, FEMNIST \cite{caldas2018leaf}, and CIFAR10 \cite{krizhevsky2009learning}.
Table \ref{tb:datasets} summarizes statistics of the datasets. Detailed descriptions of the datasets are deferred to Appendix \ref{app:dataset}.
\begin{table}[h]
\centering
\vspace{-0.3cm}
\caption{Dataset statistics.}
\fontsize{8}{7}\selectfont
\begin{tabular}{@{}l|llll@{}}
\toprule[1pt]
Datasets            & MNIST                                                                                                 & FMNIST                                                                                                & FEMNIST &CIFAR10\\ \midrule
\# training samples & 60,000                                                                                                & 60,000                                                                                                & 24,805  &50,000\\
\# testing samples  & 10,000                                                                                                & 10,000                                                                                                & 5,401  & 10,000\\
\# features         & 784                                                                                                   & 784                                                                                                   & 784   &  1024\\
\# classes          & 10                                                                                                    & 10                                                                                                    & 62   &   10\\
Partition methods   &  $ Dir(0.1)$ & $ Dir(0.1)$ & Natural &$ Dir(0.1)$\\ \bottomrule[1pt]
\end{tabular}
\label{tb:datasets}
\vspace{-0.25cm}
\end{table}

\textbf{Metrics.}
We use Accuracy as the model performance metric of downstream tasks and take Trading Volumes as the metric to measure the economic benefits of the marketplace. For the metric Accuracy, it is calculated locally and reported as the average value of all model buyers' on the market. We report the results on the test environment of the RL solver, \ie after training. We initialize the training and test for 10 times and report the mean and std of the runs.

\textbf{Implementation Details.}
\label{sec:imp}
The federated learning algorithm is implemented as the celebrated FedAvg \citep{mcmahan2017communication}. To realize the RL-based auction framework, both policy net and value net are implemented as a 3-layer multi-layer perception (MLP) with embedding dimensions 1024, 256, and 64. The default number of active sellers in each auction round is set as 70\% of the total number of clients, and $k$ is set as 5. Utility $u_i$ of clients are sampled from a uniform distribution between $(0,1)$. Due to limited space, details of the setting of other hyper-parameters are deferred to Appendix \ref{sec:imp_detail}.



\textbf{Baselines}. To the best of our knowledge, no previous work has been proposed for the federated marketplace. To evaluate the effectiveness of the proposed approach, we include the widely used general second price (GSP) auction as a baseline. To adapt GSP to FL market, we draw a multi-object auction with price discrimination. That is, for a seller's model, top-k highest-priced buyers win the auction, and pay the second highest price. For example, if buyers $\{c_1, c_2, c_3, c_4, c_5\}$ reports bidding values $\{35,22,13,11,1\}$ for a seller's model, and $k=2$, then buyers $c_1$ and $c_2$ win the model and pays $\{22, 13\}$, respectively. 

\subsection{Main Results}



\begin{figure} [t]
\subfigure[Trading Volumes (MNIST)]{
\centering
\includegraphics[width=0.5\linewidth]{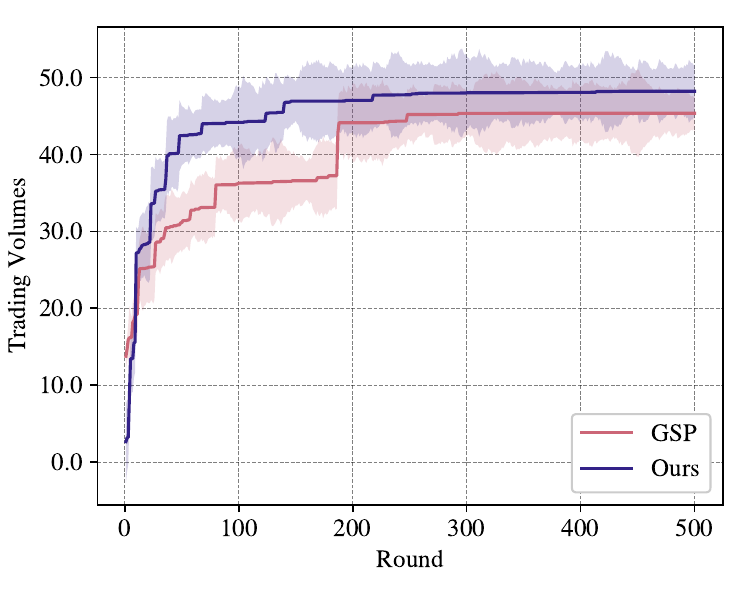}
}%
\centering
\subfigure[Accuracy (MNIST)]{
\centering
\includegraphics[width=0.5\linewidth]{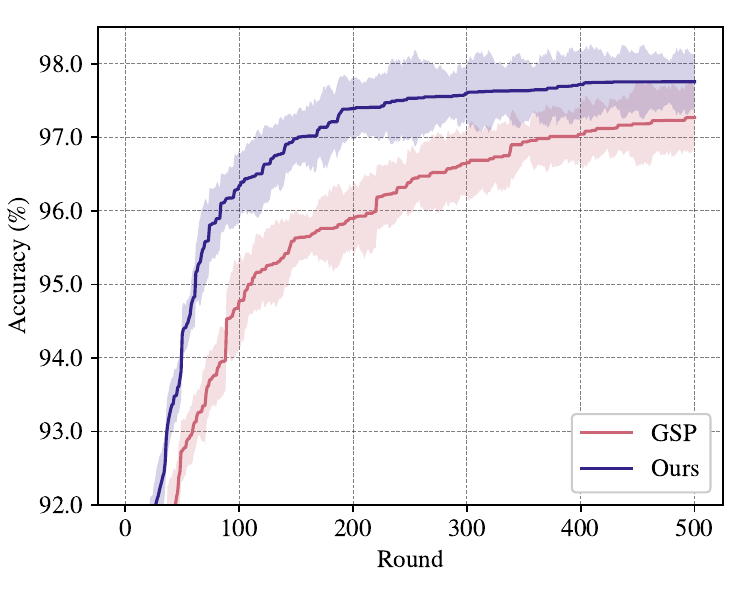}
}%

\vspace{-0.15cm}
\caption{Compare with the baseline on Trading Volumes and downstream task Accuracy on MNIST.}
\vspace{-0.35cm}
\label{fig:main}
\end{figure}

To verify the effectiveness of the proposed approach, we compare it with the baseline GSP on four datasets and record Trading Volumes and downstream task Accuracy of the test environment. The results on MNIST dataset are plotted in Figure \ref{fig:main}. Mean value and 95\% confidence intervals are reported. Due to limited space, results on FMNIST, FEMNIST, and CIFAR10 are deferred to Appendix \ref{app:exp_main_results}. It can be observed that the proposed method (denoted as Ours) outperforms GSP across all datasets, especially for the metric trading volumes. The convergence speed of the proposed approach is also faster than the baseline. This is because the RL markting operation solver on the server side can monitor the market and take better allocation actions, considering both bidding values and the state of the marketplace. In contrast to GSP, where only the willingness of buyers is considered, the proposed approach offers a global perspective from the server side to optimize the allocation and pricing functions, resulting in better performance in terms of revenue and accuracy. More detailed discussions on the main results can be found in Appendix \ref{app:exp_main_results}.

\subsection{The FL Market's Robustness to Bidding Strategies} 

Although the design of clients' bidding strategy is beyond the scope of this work, we recognize that different auction environments may follow different bidding strategies. It is essential for a market design to accommodate various bidding strategies. To this end, we simulate the auction on three bidding strategies. 

$\bullet$ \textit{Stochastic}. A buyer reports her bidding \wrt all sellers' models following certain probabilistic distribution.

$\bullet$  \textit{Greedy}, where, for the first several rounds, a buyer bids using the stochastic strategy while maintaining two lists 1) a revenue list, in which the $i$-th entry is calculated as the summation of revenue in previous auction rounds when seller $i$ is successfully bidden; and 2) a winning count list, in which the $i$-th entry is the number of times when the client wins seller $i$'s model. Based on the two lists, in subsequent auction rounds, the bidding values are determined as softmax of revenue over winning counts multiplying $u_i$. 

$\bullet$ \textit{$\epsilon$-Greedy}, which also begins with stochastic bidding, but in subsequent rounds, a model buyer samples $p\in (0,1)$, and if $p$ is greater than $\epsilon$, the bidding values are determined the same as Greedy, otherwise, the same as Stochastic. The stochastic bidding rounds in greedy-based strategies is set as 20\% of the total number of FL communication rounds, and $\epsilon$ is set as 0.1.
 

\begin{table}[t]
\vspace{-0.15cm}
\caption{Effect of bidding strategies.}
\fontsize{8}{7}\selectfont
\centering
\begin{tabular}{@{}c|c|c|ll@{}}
\toprule
Dataset                  & Method                & \begin{tabular}[c]{@{}c@{}}Bidding\\ Strategy\end{tabular} & \multicolumn{1}{c}{\begin{tabular}[c]{@{}c@{}}Accuracy\\ (\%)\end{tabular}} & \multicolumn{1}{c}{\begin{tabular}[c]{@{}c@{}}Trading\\ Volumes\end{tabular}} \\ \midrule
\multirow{6}{*}{MNIST}   & \multirow{3}{*}{GSP}  & Stochastic                                                 & 97.27$\pm$0.27                                                              & 45.37$\pm$4.53                                                                \\
                         &                       & Greedy                                                     & 97.12$\pm$0.04                                                              & 65.97$\pm$4.71                                                                \\
                         &                       & $\epsilon$-Greedy                                          & 97.87$\pm$0.19                                                              & 61.88$\pm$4.40                                                                \\ \cmidrule(l){2-5} 
                         & \multirow{3}{*}{Ours} & Stochastic                                                 & 97.76$\pm$0.20                                                              & 48.22$\pm$5.37                                                                \\
                         &                       & Greedy                                                     & 97.71$\pm$0.21                                                              & 66.00$\pm$3.67                                                                \\
                         &                       & $\epsilon$-Greedy                                          & 97.41$\pm$0.13                                                              & 93.20$\pm$3.25                                                                \\
 \bottomrule
\end{tabular}
\label{tb:bid}
\vspace{-0.35cm}
\end{table}

Results on MNIST dataset are shown in Table \ref{tb:bid} and others are deferred to Appendix \ref{app:exp_bid}. It is observed that the proposed approach outperforms GSP in most cases, which verifies its robustness. It is worth mentioning that the advantage of the proposed method is more significant under greedy-based bidding strategy. This is because clients' bidding behavior is more stable, which reduces the dynamics of the marketplace and makes it easier for the policy and value networks to train. Another interesting observation is that both greedy-based strategies achieve higher cumulative reward while lower accuracy compared with Stochastic. This can result from the increased number of successful biddings when clients tend to bid towards certain others. However, it is not yet guaranteed that the obtained model can make a positive contribution to the client's performance. 

\subsection{Social Fairness Analysis}

\begin{figure} [t]
\centering
\subfigure[Overall Trading Volumes]{
\centering
\includegraphics[width=0.47\linewidth]{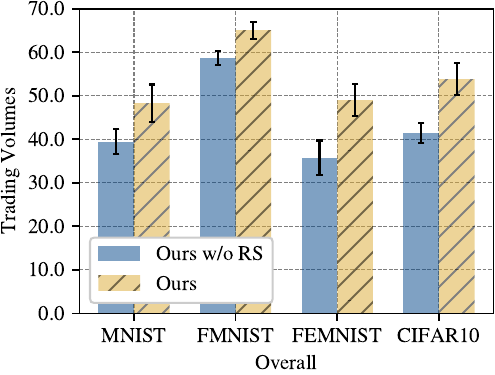}
}%
\subfigure[Bottom 10\% Trading Volumes]{
\centering
\includegraphics[width=0.47\linewidth]{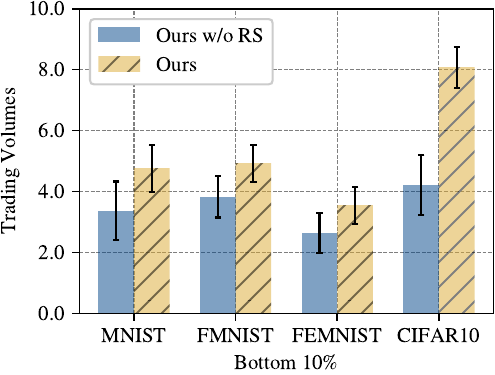}
}%
\\
\subfigure[Overall Accuracy]{
\centering
\includegraphics[width=0.47\linewidth]{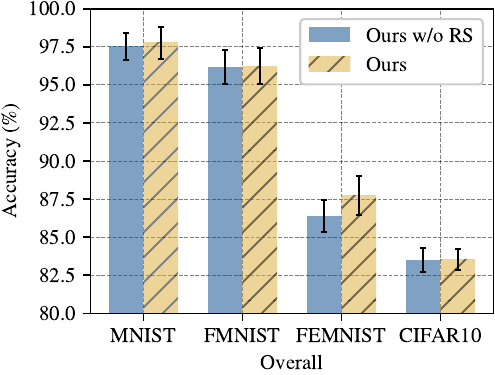}
 }%
\subfigure[Bottom 10\% Accuracy]{
\centering
\includegraphics[width=0.47\linewidth]{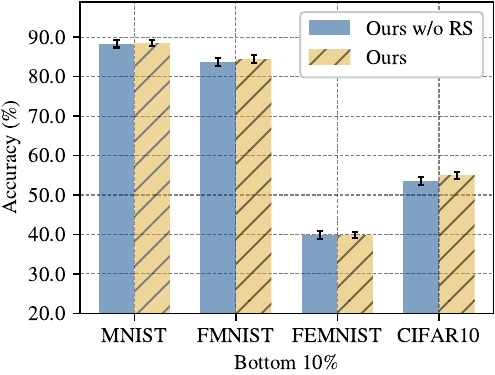}
}%
\centering
\vspace{-0.2cm}
\caption{Effect of the reward shaping mechanism.}
\vspace{-0.35cm}
\label{fig:rs}
\end{figure}

In Equation~\ref{eq:rs}, we introduce a reward shaping term to enhance social fairness, \ie the reward of a client should be higher than a predefined threshold. To verify the effectiveness of this mechanism, we conduct an ablation study. In the variant without the reward shaping term (denoted as Ours w/o RS), we remove the reward shaping term in Equation~\ref{eq:rs} and run the auction. We compare the variant and the originally proposed approach on three datasets with respect to metrics overall accuracy, trading volumes, bottom 10\% accuracy, and bottom 10\% trading volumes. The last two metrics are calculated as the average performance of the bottom 10\% clients in terms of accuracy and trading volumes, and the results are plotted in Figure \ref{fig:rs}. 

With no unfairness penalty (denoted as Ours w/o RS), it can be observed that the variant has worse trading volumes on bottom clients, although its overall trading volumes is compatible with the original approach. A similar phenomenon can be found in accuracy. This confirms the effectiveness of the reward shaping term. In addition, we find that the fairness term has a more significant effect on revenue compared with accuracy. The relative improvements after applying reward shaping are $\sim 40\%$, $\sim 0.5\%$ in Figure~\ref{fig:rs} (b), (d) respectively. This can be attributed to the fairness term has a more significant effect on revenue compared to accuracy because it specifically targets monetary equality of the allocation process.

Besides, we also conduct experiments regarding key hyper-parameters of the FL market: number of model copies, buyer utility distribution, and the ratio of authorized sellers. Due to limited space, the results are deferred to Appendix \ref{app:exp_hyper_param}.

\section{Related Work}
\label{sec:related}
\subsection{Federated Learning}
Depending on how data is divided in terms of samples and features, FL can be categorized into three primary groups: horizontal federated learning (HFL) \cite{li2022federated,shi2023towards,zhu2021federated}, vertical federated learning (VFL) \cite{liu2022vertical,wei2022vertical,li2023vertical}, and Federated Transfer Learning (FTL) \cite{yang2019federated}. HFL can be further categorized into cross-silo \cite{li2022federated} and cross-device setting, depending on how large amount of data each client possesses. Our work falls into the category of HFL and cross-silo settings. The possibility of isolated training on institutions that own a relatively large amount of data brings value to the local models/data and that well suits the nature of the marketplace.

\subsection{Incentive Mechanism Design in FL}
The majority of research on incentive mechanisms in cross-device scenarios can be categorized into three categories: auction-based ~\citep{pei2020survey,incentive2022survey}, contract theory-based~\citep{Huang20contract,Lim21IoVcontract,Ye20contract,Ding21contract}, and game-theory-based mechanisms~\citep{Weng21ng_FedServing,Radanovic14ng,tahanian2021game,Yu18sg,Zhan20sg,Feng19sg,Pandey19sg,Hu20sg,liu2020fedcoin}. However, only a few of existing works focus on the incentive mechanism design in cross-silo FL. These studies are mainly based on game theory to determine the best reward allocation scheme. In~\citep{Tang21cross_silo}, the authors design the incentive scheme for cross-silo FL from the public goods aspect, and formulate the mechanism design problem as a social welfare maximization problem. \citep{Song19Shapley} adapts Shapley value to measure the contribution of the FL participants and allocate the profit generated by the trained FL model. In~\citep{Han22tokenCrossSilo}, the authors propose a tokenized incentive mechanism, which adopts the tokens to pay for the clients' training services at each FL round. Compared to these existing works, our proposed mechanism introduces auction, which allows clients to trade their updated models during FL communication rounds. This enhances the flexibility and the effectiveness of the model pricing in cross-silo FL.

\section{Conclusion and Future Work}
\label{sec:conclusion}
In this paper, we propose to view the model exchanging in FL as trading to increase liquidity of models and fairly reward the model owners. We take the view from the side of the trading platform and design an auction-based FL market. We propose to price local model updates of clients based on performance gain and introduced an RL-based framework as marketing operation solver, aiming at maximizing trading volumes while satisfying social fairness and efficient allocation constraints. Additionally, we establish theoretical contributions to incentivise true valuation revealing of buyer clients. Experiments on four real-world datasets show the effectiveness of the proposed approach. 

\textbf{Limitations.} However, the proposed approach has several limitations: 1) automatic bidding can be studied and applied to maximize individual buyer clients' benefits; and 2) the number of model copies selling by each seller may be different across clients/rounds. An allocation function with an authorizing module can be a possible solution. Nevertheless, we maintain a positive view of the impact of this paper and hope it serves as a foundational work on FL markets and inspires future research.


\section*{Broader Impact}

The approach proposed, which encompasses an auction-based marketplace for model exchanges and a reinforcement learning-based framework for market operations, has the potential to affect the societal structure of collaborative machine learning by promoting equitable compensation for data contributions and model improvements. 

Ethically, our approach attempts to enforce social fairness and efficient allocation of resources, which may help mitigate some of the biases that can be present in machine learning. By rewarding contributions based on performance gain, there is a tangible incentive to provide high-quality data and model updates, which could lead to overall improvements in the fairness and accuracy of FL systems.
\bibliography{refs}
\bibliographystyle{icml2024}

\newpage
\appendix
\onecolumn
\section{Summary of Notations}
\label{app:notation}
\begin{table}[h]
\centering
\caption{Summary of notations.}
\label{tb:notations}
\begin{tabular}{@{}ll@{}}
\toprule[1pt]
Notation          & Description                                                                                                                                  \\ \midrule
\rowcolor[HTML]{FFFFFF} 
$i, N$            & Index and total number of clients                                                                                                            \\ \midrule
$t,T$             & Index and total round of biddings                                                                                                            \\ \midrule
\rowcolor[HTML]{FFFFFF} 
$u_i$             & Utility of a unit performance increase of client $i$                                                                                         \\ \midrule
\rowcolor[HTML]{FFFFFF} 
\rowcolor[HTML]{FFFFFF} 
$b_{i,j}^{t}$   & \begin{tabular}[c]{@{}l@{}}Client $i$'s bidding to client $j$'s model for the \\  received dividend of a unit performance increase\end{tabular} \\ \midrule
\rowcolor[HTML]{FFFFFF} 
$\mathcal{M}_i^t$ & \begin{tabular}[c]{@{}l@{}}Downstream task performance of client $i$ \\ at round $t$\end{tabular}                                            \\ \midrule
\rowcolor[HTML]{FFFFFF} 
$A_i^t$           & \begin{tabular}[c]{@{}l@{}}Model uploaded to the server of client $i$ \\ at round $t$\end{tabular}                                           \\ \midrule
\rowcolor[HTML]{FFFFFF} 
$G_i^t$           & Performance gain of client $i$ at round $t$                                                                                                  \\ \midrule
\rowcolor[HTML]{FFFFFF} 
$\Delta G_i^t$    & \begin{tabular}[c]{@{}l@{}}Difference of performance gain of client $i$ \\ between round $t$ and $t-1$\end{tabular}                          \\ \midrule
\rowcolor[HTML]{FFFFFF} 
$p_{i,j}^t$       & Price paid by client $i$ to client $j$ at round $t$                                                                                          \\ \midrule
\rowcolor[HTML]{FFFFFF} 
$a_{i,j}^t$       & Allocation score of client $i$'s bidding to client $j$                                                                                       \\ \midrule
 $R_i^t$      & Client $i$'s reward at round $t$\\                                                                                         
\bottomrule[1pt]
\end{tabular}
\end{table}

\section{Proofs}
\subsection{Proof of Theorem \ref{th:copy}}
\label{app:proof1}

\begin{theorem} [\textbf{Theorem \ref{th:copy} Restate}]
When each model seller has no limits on the number of the sold model copies, no model buyer will have the incentive to submit a bid equaling its predicted monetary utility unless there exists a constraint on the number of model copies:
\begin{equation}
\left\{
\begin{aligned}
&b_{i,j}^t < u_i^t(X_{i,j}^{t-1}) , \quad \text{if} \quad k = N;   \\
&b_{i,j}^t \leq u_i^t(X_{i,j}^{t-1}), \quad \text{if}\quad k < N.
\end{aligned}
\right .
\end{equation}

\label{th:copy_restate}
\end{theorem}

\begin{proof}
We first assume after a certain FL period, each model buyer will acquire sufficient information of the model seller through exploration bidding.

\begin{equation}
    |X_j^{t} - X_j^{t-1}| <\epsilon, \quad \text{for} \quad  j\in N.
\end{equation}

Then, each model buyer will only submit its bid on its highest expected revenue and receive the model copy that it bids for with 100\% certainty as we do not limit the number of model copies ($k=N$).
\begin{equation}
\left\{
\begin{aligned}
     &b_{i,j}^t = \tau < u_i^t(X_{i,j}^{t-1}), \quad &\text{if } j \in \mathop{\arg\max}\limits_Z \sum_{z\in Z} u_i^t(X_{i,z}^{t-1})-  b_{i,z}^t, Z \subset N; \\
     &b^t_{i,j} =0,  \quad &\text{otherwise}.
\end{aligned} 
\right .
\end{equation}

Normally, different model buyers adopt the maximal model increase from different model sellers, and even in the worst cases that all buyers only bid for one model seller, it will still bid at a small $\tau$ as it will always receive the model copy.

We then prove the existence of a model buyer who will bid equaling to his expected utility when the model seller has a limit model copies $k<N$ after the exploration period, \ie the condition for equality between $\tau$ and $u_i^t(X_{i,j}^{t-1})$. 

In the worst case mentioned above, when all model buyers achieve a certain model increase from one specific model seller $\hat{j}$ while obtaining no improvement from other models, all model buyers will bid for this model $A_{\hat{j}}^{t-1}$ in the round $t$.
\begin{equation}
\left\{
\begin{aligned}
     &u_i^t(X_{i,j}^{t-1}) >0, \quad  \text{if} \quad j = \hat{j}; \\
     &u_i^t(X_{i,j}^{t-1}) =0, \quad \text{otherwise}.
\end{aligned} 
\right .
\end{equation}

Then, to become the $k$ winners out of $N$ bidders, the buyer clients have the incentive to bid at most truthfully, \ie

\begin{equation}
\left\{
\begin{aligned}
     &b_{i,j}^t  =\tau \leq u_i^t(X_{i,j}^{t-1}), \quad &\text{if } j = \hat{j}; \\
     &b^t_{i,j} =0, \quad &\text{otherwise}.
\end{aligned} 
\right .
\end{equation}

This completes the proof.

\end{proof}

\subsection{Proof of Theorem \ref{th:curse}}
\label{app:proof2}

\begin{theorem}[\textbf{Theorem \ref{th:curse} Restate}]
Suppose the clients authorized as model sellers at each round are determined by a selective mechanism on the server. If all clients are authorized to sell at each round, model buyers will not bid truthfully. But there exists a selection mechanism $S$ that can urge a sufficient number of model buyers to bid truthfully. 
\label{th:curse_restate}
\end{theorem}

\begin{proof}
In the last part of the proof of Theorem \ref{th:copy}, one may notice that a model buyer will stick to models it will successfully bid and has positive utility in reward, denoted as below. 
\begin{equation}
\left\{
\begin{aligned}
     & b_{i,j}^t  \leq u_i^t(X_{i,j}^{t-1}), & \text{if } j \in \{j|P_j(b_{i,j}^t)*u_i^t(X_{i,j}^{t-1})>0\}; \\
     & b^t_{i,j} =0, \quad &\text{otherwise}.
\end{aligned} 
\right .
\end{equation}
where $P_j$ denotes the possibility that it is only $k$ part of the bidders who submit large than $0$ bid for the model $j$. 

As all models are accessible, model buyers may learn a near-optimal bidding strategy that they only bid for those models with sufficient performance increase (denoted as utility) with a sufficient small bid $b_{i,j} \sim 0$, and finally all model buyers bidding strategy will converge due to the stable performance increase prediction feature $X_j$ and no one will change his target bidding model identification $j^T=j$ in the latter $T$ round.

However, when at each round, the set of clients whose model can be sold is determined by the host with a certain selective mechanism $S$, such a situation will not exist.
Take an m-random-picking mechanism $S_{m}$ as an example, where each round started host randomly picks $m<N$ clients to sell models, and no model buyer can forecast whose model will be sold. As a result, it can never predict whether it is fortunate enough to become the $k$ part of bidders who submit large than $0$ bid for the model $j \in \{S_{m}(N)^t\} $ if $m$ is smaller enough (i.e. $m=1$) where $\{S_{m}(N)^t\}$ is the $m$ randomly selected model sellers at round $t$. Then, the selective mechanism nudges every model buyer $i$ to submit a higher bid on the target sold model $j \in \{S_{m}(N)^t\}$ than the bid without the selective mechanism. Mathematically,
\begin{equation}
    b_{i,j}^t | S_{m} \quad \geq b_{i,j}^t,\quad j \in \{S_{m}(N)^t\}.
\end{equation}

That is, when authorized sellers vary in each auction round via a selection mechanism $S$, e.g., random selection, it will urge a sufficient large number of model buyers to bid as high as possible in round $t$ with accessible model $j \in \{S(N)^t\}$, where $\{S(N)^t\}$ denotes the selected sellers at round $t$:
\begin{equation}
    b_{i,j}^t | S = u_i^t(X_{i,j}^{t-1}) , j \in \{S(N)^t\}.
\end{equation}
This completes the proof.
\end{proof}

\section{Training of the RL-based Solver}
\label{app:rl_solver}

Pseudo code and pipeline of the training of the RL-based solver is in Algorithm \ref{alg:training} and Figure \ref{fig:rlow_rl}.

\begin{figure*}[t] 
\centering
\includegraphics[width=0.8\linewidth]{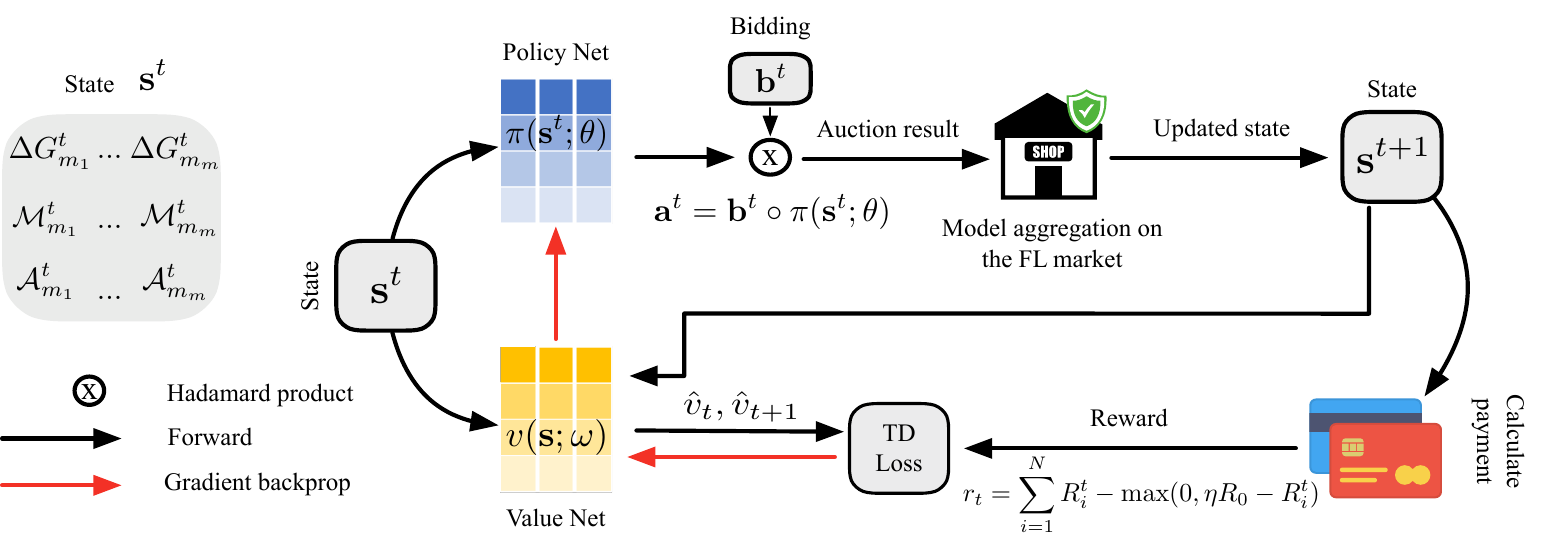}
\caption{The reinforcement learning based solution for auction-based FL market.}
\label{fig:rlow_rl}
\end{figure*}
\begin{algorithm}[t]
\caption{Training of the RL-based FL Market} 
\begin{algorithmic}[1]
\REQUIRE 
{$N$ clients (participants of the FL market); the performance of isolated training model on each client: $\{M_{0,i}\}$;}
\STATE {\textbf{Initialize:} Initialize model weights of the policy net $\theta$ and value net $\omega$; assign $t$ as 0: $t\leftarrow 0$.}
\WHILE {$t<T$:}
    \STATE \textcolor{gray}{/*On the side of the server*/}
    \STATE {The server randomly authorizes a subset of sellers for selling, denoting the number of authorized sellers as $m$;}
    \STATE \textcolor{gray}{/*On the side of the seller clients*/}
    \STATE {The authorized seller clients announce model performance and upload model weights to the server, denoted as $\{\mcM_{m_i}^t\}$ and $\{A_{m_i}^t\}$;}
    \STATE \textcolor{gray}{/*On the side of the buyer clients*/}
    \STATE {The buyer clients report bidding prices to the server, denoted as $\{\mbb^t_{k_j}\}$;}
    \STATE \textcolor{gray}{/*On the side of the server*/}
    \STATE {Formulate the state of the market $\mbs^t \leftarrow [\mbs_{m_1}^t,...,\mbs_{m_m}^t]$;}
    \STATE {Calculate action score by Equation \ref{eq:mono}, determining the winner buyers who has the top-k randed score;}
    \STATE {Performing model aggregation and delivering the aggregated models to winner buyers, obtaining $\mbs^{t+1}$ from the reported performance gain of clients;}
    \STATE {Calculate $\hat{v}_t \leftarrow v(\mbs^t;\omega), \hat{v}_{t+1} \leftarrow v(\mbs^{t+1};\omega)$ as in Equation \ref{eq:value};}
    \STATE {Performing TD-based model update according to Equation \ref{eq:td} and \ref{eq:update};}
    \STATE {$t\leftarrow t+1$;}

\ENDWHILE
\end{algorithmic}
\label{alg:training}
\end{algorithm}
\section{Privacy Analysis}
\label{app:privacy}
In the proposed FL market, the information transmitted between clients and the server includes model parameters and model performance. However, as submitting model parameters to the central server is consistent with the classic FL paradigm FedAvg \cite{mcmahan2017communication}, and the model performance value contains zero information about sensitive information of clients' local data, the FL market brings no additional privacy threats. This is because though model performance is a function of local data, to the best of our knowledge, there is no evidence that the function is invertible. That is, it is infeasible to deduce the local data merely from model performance. 
\section{Detailed Dataset Description}
\label{app:dataset}
\begin{itemize}
    \item \textbf{MNIST} \footnote{\url{http://yann.lecun.com/exdb/mnist/}}: MNIST is a handwritten digits recognition dataset with 60,000 training samples and 10,000 testing samples. Each sample is a 28x28 grayscale image containing a digit from 1-10.


 \item \textbf{FMNIST} (Fashion-MNIST) \footnote{\url{https://github.com/zalandoresearch/fashion-mnist}}: The FMNIST is an MNIST-like fashion product dataset with 10 categories. The train/test splitting ratio, the image size, and the data format are the same as MNIST.


 \item \textbf{FEMNIST} (Federated Extended MNIST) \footnote{\url{https://leaf.cmu.edu/}}: FEMNIST is a widely adopted FL dataset constructed from the Extended MNIST (EMNIST) dataset, containing 62 categories of handwritten characters/digits. Since the dataset is collected from 3,350 writers, the federated non-IID dataset can be constructed with each writer as one client. Following the settings in~\citep{caldas2018leaf}, a subset of the dataset with 50 clients is created, and in each client, the dataset is split into training, validation, and testing sets with a ratio of 6:2:2. 

 \item \textbf{CIFAR10} \footnote{\url{https://www.cs.toronto.edu/~kriz/cifar.html}}: 
The CIFAR10 dataset is a collection of 60,000 32x32 color images in 10 different classes, with 6,000 images per class. The classes include common objects such as airplanes, automobiles, birds, cats, deer, dogs, frogs, horses, ships, and trucks. The dataset is divided into 50,000 training images and 10,000 testing images.
 \end{itemize} 
FEMNIST dataset has natural federation partition. For MNIST, FMNIST, and CIFAR10 datasets, following the heterogeneous partition manners used in \citep{li2022federated,marfoq2021federated,caldas2018leaf,diao2020heterofl,hsu2019measuring}, we apply Dirichlet allocation with Dirichlet factors equals 0.1 to split these datasets into 100 clients.

\section{Other Implementation Details}
\label{sec:imp_detail}

All experiments are implemented on 2 Intel(R) Xeon(R) Gold 6248R CPU @ 3.00GHz and 8 NVIDIA A100.

The federated learning algorithm is implemented as the celebrated FedAvg \citep{mcmahan2017communication}. As for the image classification network, following \citep{mcmahan2017communication}, we implement a 2-layer multi-layer perception (MLP) for MNIST and FMNIST datasets, with embedding size set as 200. For FEMNIST, also following \citep{mcmahan2017communication} we use a convolutional neural network with 2 hidden layers. The embedding dimension is set as 32 and 64, respectively, with ReLU activations and kernel-size-2 max pooling layers followed. 

To realize the RL-based auction framework, both policy net and value net are implemented as a 3-layer multi-layer perception (MLP) with embedding dimensions 1024, 256, and 64. For the model parameter component in $\mbs^t$, instead of using the whole model of a client, we take model weights of the last layer as representative. We take $\alpha$ and $\beta$ in Equation \ref{eq:update} as 5e-4. The default number of active sellers in each auction round is set as 70\% of the total number of clients, and $k$ is set as 5. Utility $u_i$ of clients are sampled from a uniform distribution between $(0,1)$. We define $\eta R_0$ as an adaptive value calculated as 0.5\% of the most rewarded clients. The total number of training rounds of the RL framework is set as 200 on all datasets. We use different random seeds to distinguish training, evaluation, and testing environments. All results are on the testing environment if not specified. We take isolated training \citep{yao2022federated} Accuracy as $\mcM_0$. We set the number of auction rounds of the evaluation and testing environment as 500 for MNIST, FMNIST and CIFAR10, 100 for FEMNIST. This setting ensures that the auction is converged. 



Since the proposed FL market focuses on the design of the market rather than the optimization of the bidding strategy, we suppose the utility rate of clients are under normal distribution $\mathcal{N} (0,1)$ so as the bidding values. After the sampling of bidding values, we normalize them according to the utility rate to meet the IR assumption. 

\section{Full Version of Related Work}
\label{app:related_work}
\subsection{Federated Learning}

Federated learning (FL) techniques maintain data confidentiality by ensuring that raw data remains at its source. Depending on how data is divided in terms of samples and features, FL can be categorized into three primary groups: horizontal federated learning (HFL) \cite{li2022federated,shi2023towards,zhu2021federated}, vertical federated learning (VFL) \cite{liu2022vertical,wei2022vertical,li2023vertical}, and Federated Transfer Learning (FTL) \cite{yang2019federated}. HFL supports cooperative training among clients possessing different data samples but identical feature spaces. VFL allows for the training of clients' datasets with shared samples but distinct features. FTL addresses situations where both sample and feature spaces differ in clients' datasets. In addition to enhancing FL models' performance and efficiency, their commercialization has gained growing interest. HFL can be further categorized into cross-silo \cite{li2022federated} and cross-device setting, depending on how large amount of data each client possesses. Our work falls into the category of HFL and cross-silo settings. The possibility of isolated training on institutions that own a relatively large amount of data brings value to the local models/data and that well suits the nature of the marketplace.

\subsection{Incentive Mechanism Design in FL}

In recent years, the majority of research on incentive mechanisms in cross-device scenarios can be categorized into three categories: auction-based mechanism, contract theory-based mechanism, and game-theory-based mechanism.

\textbf{Auction-based:} Auction-based incentive mechanism involves the creation of an auction where the clients (edge device) bid for the right to participate in FL~\citep{pei2020survey,incentive2022survey}. 
For example, in~\citep{deng2021fair}, the authors propose a quality-aware auction mechanism that evaluates the contribution of the edge devices to the FL process through the reduction in loss. 
In~\citep{Le20VCG}, a Vickrey-Clarke-Groves (VCG) based auction mechanism is proposed to find the optimal client selection and payment determination. 

\textbf{Contract-Theory-based:} The objective of the contract theory-based incentive mechanism is to construct the optimal contract that can maximize the payoff or the utility of the server~\citep{Huang20contract}.
In~\citep{Lim21IoVcontract}, the authors applied a multi-dimensional contract in the UAV-Enabled Internet of Vehicles (IoV) scenario. 
Meanwhile, in~\citep{Ding21contract}, the authors take the data size and the communication time as the contract decision variables and explore the optimal contract design under different information cases: complete information, weakly incomplete information, and strongly incomplete information cases. 

\textbf{Game-Theory-based:} The game-theory-based methods can be mainly divided into two types: non-cooperative game-based mechanism and Stackelberg game-based mechanism. In the non-cooperative game-based mechanism, each participant aims to maximize their own utility, and their strategies are decided simultaneously~\citep{Weng21ng_FedServing,Radanovic14ng,tahanian2021game}. 
On the other hand, the Stackelberg game-based mechanism, also known as the leader-follower game, operates under the assumption that the leader participants decide their strategies first based on the expected strategies of the follower participants. The followers then decide their strategies after observing the actions of the leaders~\citep{Yu18sg,Zhan20sg,Feng19sg,Pandey19sg,Hu20sg}.

In the cross-silo FL scenario, the design of an effective and efficient incentive mechanism is critical for encouraging the participation of high-quality institutions~\citep{pei2020survey,incentive2022survey}. However, only a few of the existing works focus on the incentive mechanism design in cross-silo FL. These studies are mainly based on the game theory to determine the best reward allocation scheme. In~\citep{Tang21cross_silo}, the authors design the incentive scheme for cross-silo FL from the public goods aspect, and formulate the mechanism design problem as a social welfare maximization problem. In~\citep{Song19Shapley}, the Shapley value is adopted to measure the contribution of the FL participants and allocate the profit generated by the trained FL model. In~\citep{Han22tokenCrossSilo}, the authors propose a tokenized incentive mechanism, which adopts the tokens to pay for the clients' training services at each FL training round. 


\section{Additional Experiment Results}
\subsection{Main Results on More Datasets}
The comparison results of Ours and GSP on FMNIST, FEMNIST, and CIFAR10 are shown in Figure \ref{fig:main_app}. We make several other observations. We observe an overtaken phenomenon, from the proposed approach to GSP, on trading volume of MNIST, FMNIST and CIFAR10. This indicates GSP-based allocation enhances the chance of obtaining high revenue models in the very first few rounds. This can be interpreted as GSP purely considers the bidding values of models, which is of advantage in leading to high revenue initially. However, when the market pattern is formalized, an optimized allocation function considering post-FL performance gain is more beneficial.

\label{app:exp_main_results}
\begin{figure*} [t]

\subfigure[Trading Volumes (FMNIST)]{
\centering
\includegraphics[width=0.3\linewidth]{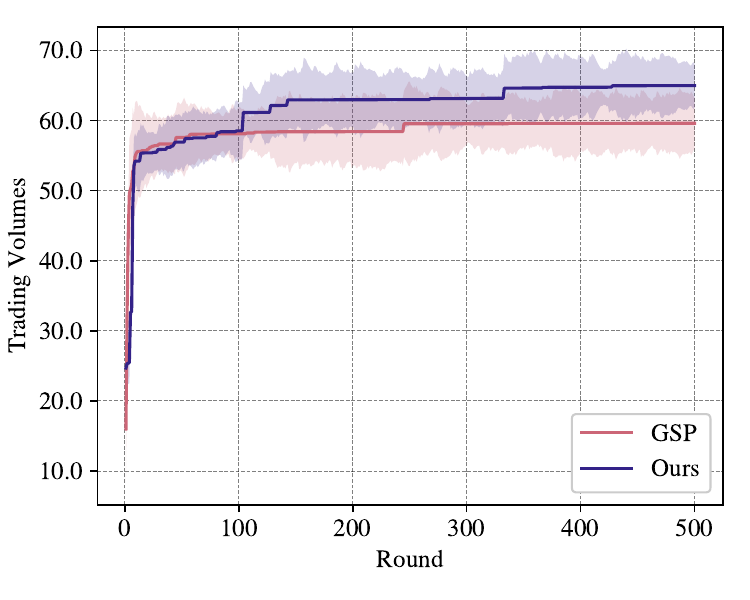}
}%
\subfigure[Trading Volumes (FEMNIST)]{
\centering
\includegraphics[width=0.3\linewidth]{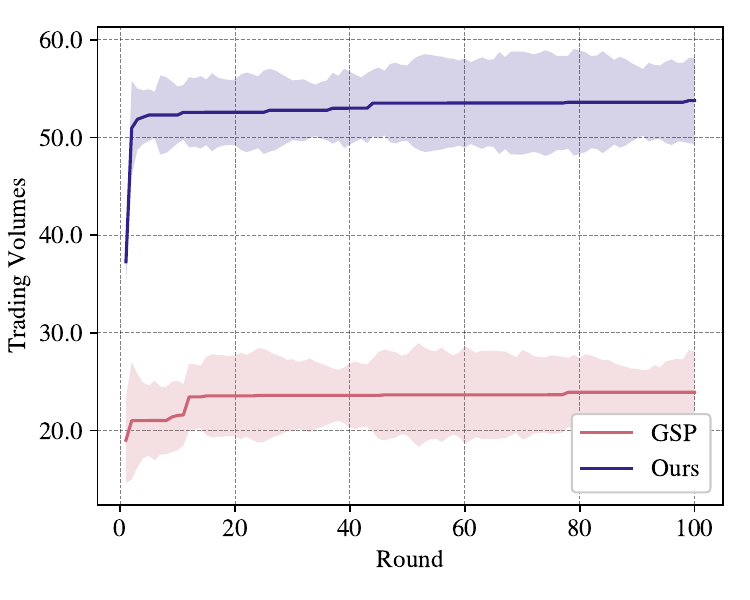}
}%
\subfigure[Trading Volumes (CIFAR10)]{
\centering
\includegraphics[width=0.3\linewidth]{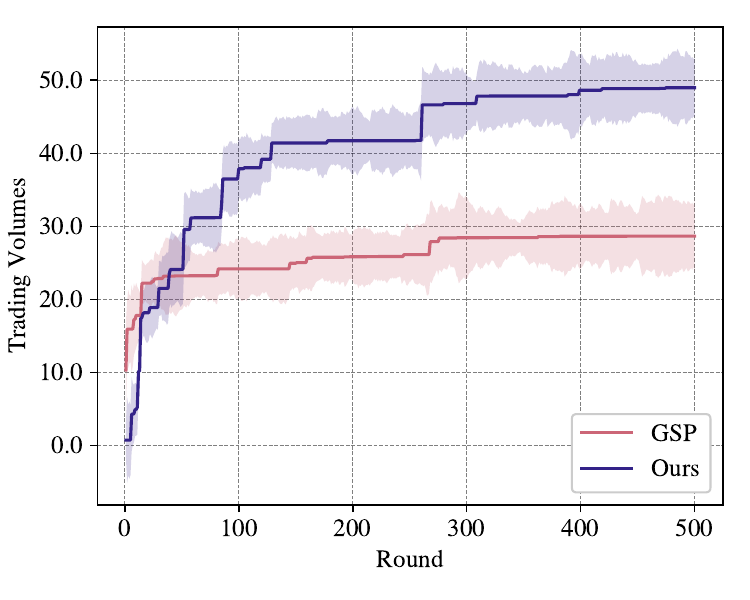}
}%
\centering
\\
\centering
\subfigure[Accuracy (FMNIST)]{
\centering
\includegraphics[width=0.3\linewidth]{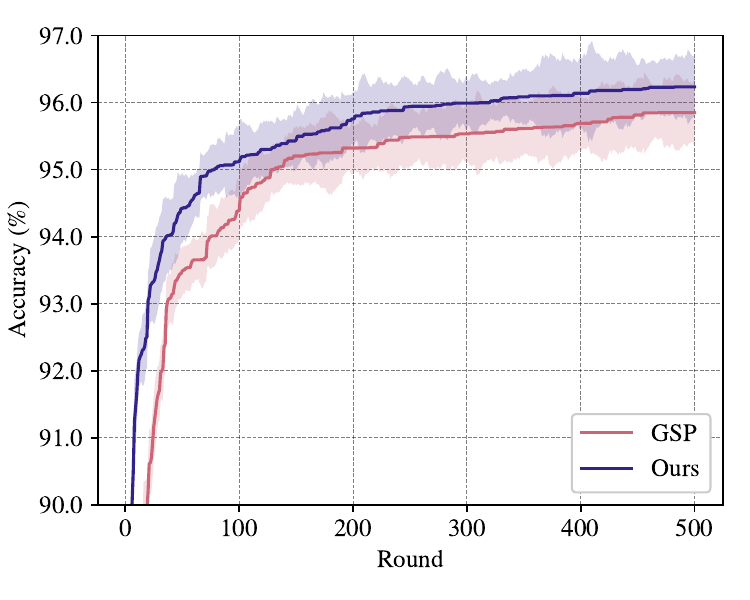}
}%
\subfigure[Accuracy (FEMNIST)]{
\centering
\includegraphics[width=0.3\linewidth]{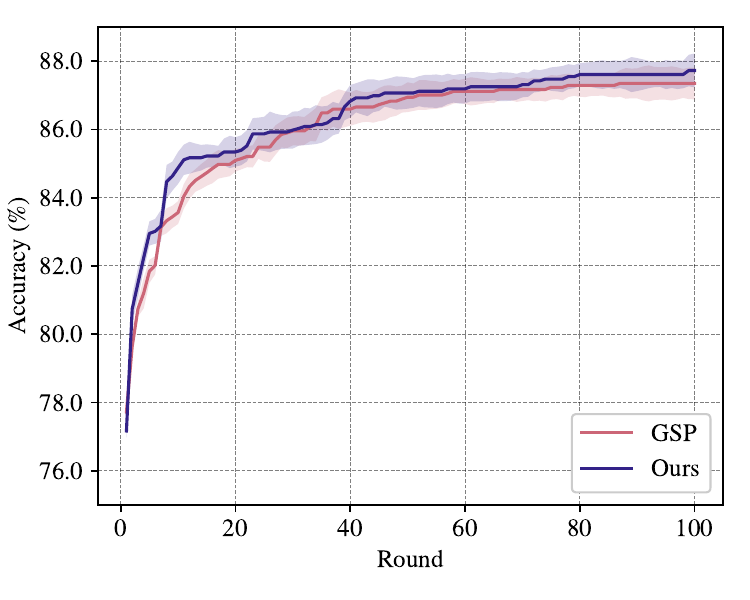}
}%
\centering
\subfigure[Accuracy (CIFAR10)]{
\centering
\includegraphics[width=0.3\linewidth]{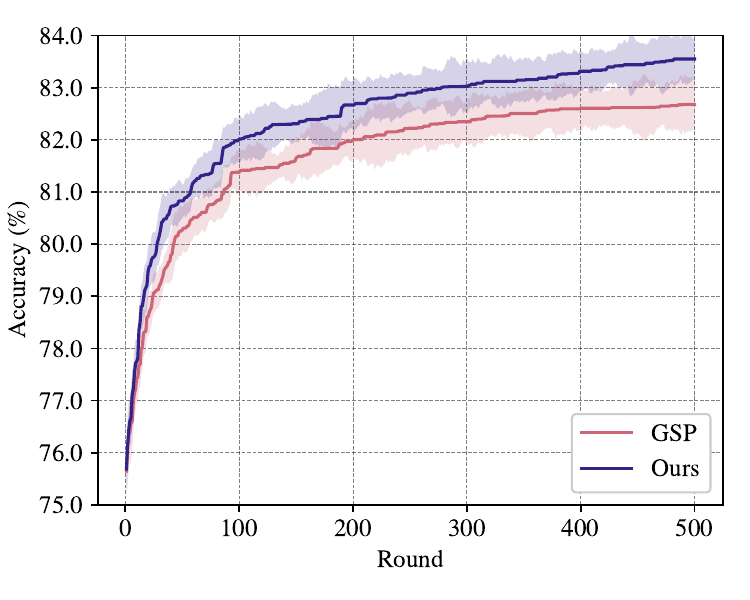}
}%
\vspace{-0.2cm}
\caption{Compare with the baseline on Trading Volumes and downstream task Accuracy.}
\vspace{-0.2cm}
\label{fig:main_app}
\end{figure*} 

\subsection{Effect of Bidding Strategies on More Datasets}
\label{app:exp_bid}
The experiment results of different bidding strategies on FMNIST, FEMNIST, and CIFAR10 are shown in Table \ref{tb:bid_more}. 

\begin{table}[]
\caption{Effect of bidding strategies.}
\fontsize{7}{7}\selectfont
\centering
\begin{tabular}{@{}c|c|c|ll@{}}
\toprule
Dataset                  & Method                & \begin{tabular}[c]{@{}c@{}}Bidding\\ Strategy\end{tabular} & \multicolumn{1}{c}{\begin{tabular}[c]{@{}c@{}}Accuracy\\ (\%)\end{tabular}} & \multicolumn{1}{c}{\begin{tabular}[c]{@{}c@{}}Trading\\ Volumes\end{tabular}} \\ \midrule
\multirow{6}{*}{FMNIST}  & \multirow{3}{*}{GSP}  & Stochastic                                                 & 95.85$\pm$0.24                                                              & 59.61$\pm$3.22                                                                \\
                         &                       & Greedy                                                     & 94.53$\pm$0.38                                                              & 66.33$\pm$5.94                                                                \\
                         &                       & $\epsilon$-Greedy                                          & 95.80$\pm$0.20                                                              & 62.41$\pm$4.42                                                                \\ \cmidrule(l){2-5} 
                         & \multirow{3}{*}{Ours} & Stochastic                                                 & 96.23$\pm$0.22                                                              & 64.98$\pm$3.31                                                                \\
                         &                       & Greedy                                                     & 94.88$\pm$0.26                                                              & 66.11$\pm$5.05                                                                \\
                         &                       & $\epsilon$-Greedy                                          & 95.98$\pm$0.19                                                              & 68.11$\pm$3.60                                                                \\ \midrule
\multirow{6}{*}{FEMNIST} & \multirow{3}{*}{GSP}  & Stochastic                                                 & 87.35$\pm$0.19                                                              & 23.92$\pm$4.36                                                                \\
                         &                       & Greedy                                                     & 87.42$\pm$0.00                                                              & 43.80$\pm$3.05                                                                \\
                         &                       & $\epsilon$-Greedy                                          & 86.86$\pm$0.13                                                              & 54.81$\pm$4.51                                                                \\ \cmidrule(l){2-5} 
                         & \multirow{3}{*}{Ours} & Stochastic                                                 & 87.73$\pm$0.16                                                              & 53.78$\pm$3.19                                                                \\
                         &                       & Greedy                                                     & 86.81$\pm$0.35                                                              & 69.43$\pm$3.50                                                                \\
                         &                       & $\epsilon$-Greedy                                          & 87.90$\pm$0.04                                                              & 64.35$\pm$1.53                                                                \\ \midrule
\multirow{6}{*}{CIFAR10} & \multirow{3}{*}{GSP}  & Stochastic                                                 & 82.71$\pm$0.09                                                              & 28.67$\pm$5.03                                                                \\
                         &                       & Greedy                                                     & 81.08$\pm$0.22                                                              & 15.18$\pm$2.81                                                                \\
                         &                       & $\epsilon$-Greedy                                          & 80.82$\pm$0.18                                                              & 27.48$\pm$2.51                                                                \\ \cmidrule(l){2-5} 
                         & \multirow{3}{*}{Ours} & Stochastic                                                 & 83.55$\pm$0.24                                                              & 48.99$\pm$4.21                                                                \\
                         &                       & Greedy                                                     & 81.11$\pm$0.16                                                              & 58.76$\pm$3.34                                                                \\
                         &                       & $\epsilon$-Greedy                                          & 81.62$\pm$0.30                                                              & 35.34$\pm$3.73                                                                \\ \bottomrule
\end{tabular}
\label{tb:bid_more}
\end{table}

\subsection{Hyper-parameter Analysis}
\label{app:exp_hyper_param}
\begin{table*}[t]
\centering
\caption{Effect of buyer utility and the ratio of authorized sellers.}
\fontsize{8}{9}\selectfont

\setlength{\tabcolsep}{1mm}{
\begin{tabular}{@{}c|c|c|cccccccc@{}}
\toprule
\multirow{2}{*}{Dataset} & \multirow{2}{*}{\begin{tabular}[c]{@{}c@{}}Utility\\ Distribution \\ of Buyers\end{tabular}} & $k=xN$ & \multicolumn{2}{c}{$x=20\%$}                                               & \multicolumn{2}{c}{$x=40\%$}                                               & \multicolumn{2}{c}{$x=60\%$}                                               & \multicolumn{2}{c}{$x=80\%$}                                              \\ \cmidrule(l){3-11} 
                         &                                                                                              & Method & Accuracy       & \begin{tabular}[c]{@{}c@{}}Trading\\ Volumes\end{tabular} & Accuracy       & \begin{tabular}[c]{@{}c@{}}Trading\\ Volumes\end{tabular} & Accuracy       & \begin{tabular}[c]{@{}c@{}}Trading\\ Volumes\end{tabular} & Accuracy       & \begin{tabular}[c]{@{}c@{}}Trading\\ Volumes\end{tabular} \\ \midrule
\multirow{4}{*}{MINIST}  & \multirow{2}{*}{$u(0,1)$}                                                                    & GSP    & 97.69$\pm$2.22 & 240.32$\pm$0.12                                           & 97.72$\pm$1.57 & 96.84$\pm$0.15                                            & 97.24$\pm$2.32 & 70.04$\pm$0.17                                            & 96.76$\pm$1.85 & 42.86$\pm$0.17                                           \\
                         &                                                                                              & Ours   & 97.99$\pm$2.52 & 183.78$\pm$0.14                                           & 98.04$\pm$1.60 & 113.90$\pm$0.35                                           & 98.04$\pm$2.23 & 69.99$\pm$0.09                                            & 97.79$\pm$2.04 & 48.64$\pm$0.15                                           \\
                         & \multirow{2}{*}{$|\mathcal{N}(0,1)|$}                                                        & GSP    & 97.73$\pm$2.29 & 95.35$\pm$0.27                                            & 98.04$\pm$2.36 & 72.63$\pm$0.13                                            & 97.34$\pm$2.14 & 41.24$\pm$0.24                                            & 97.11$\pm$2.05 & 54.41$\pm$0.08                                           \\
                         &                                                                                              & Ours   & 97.88$\pm$1.69 & 208.85$\pm$0.27                                           & 97.92$\pm$1.53 & 99.74$\pm$0.34                                            & 98.06$\pm$1.75 & 124.28$\pm$0.36                                           & 97.69$\pm$1.29 & 101.61$\pm$0.22                                          \\ \midrule
\multirow{4}{*}{FMNIST}  & \multirow{2}{*}{$u(0,1)$}                                                                    & GSP    & 96.07$\pm$1.25 & 260.73$\pm$0.27                                           & 96.50$\pm$0.98 & 73.25$\pm$0.31                                            & 96.12$\pm$1.96 & 54.07$\pm$0.33                                            & 96.20$\pm$1.60 & 58.15$\pm$0.17                                           \\
                         &                                                                                              & Ours   & 96.38$\pm$1.33 & 302.43$\pm$0.06                                           & 96.10$\pm$1.84 & 91.43$\pm$0.23                                            & 96.15$\pm$2.23 & 83.34$\pm$0.15                                            & 96.38$\pm$1.24 & 58.80$\pm$0.22                                           \\
                         & \multirow{2}{*}{$|\mathcal{N}(0,1)|$}                                                        & GSP    & 96.24$\pm$2.35 & 296.88$\pm$0.23                                           & 95.89$\pm$1.99 & 69.27$\pm$0.27                                            & 95.95$\pm$1.75 & 54.37$\pm$0.11                                            & 95.78$\pm$1.63 & 38.61$\pm$0.19                                           \\
                         &                                                                                              & Ours   & 96.20$\pm$1.94 & 382.25$\pm$0.05                                           & 96.22$\pm$2.81 & 92.23$\pm$0.05                                            & 95.66$\pm$2.23 & 80.74$\pm$0.25                                            & 96.23$\pm$2.45 & 88.71$\pm$0.13                                           \\ \midrule
\multirow{4}{*}{FEMNIST} & \multirow{2}{*}{$u(0,1)$}                                                                    & GSP    & 89.21$\pm$1.89 & 59.13$\pm$0.21                                            & 89.12$\pm$1.54 & 40.56$\pm$0.15                                            & 88.95$\pm$2.04 & 16.02$\pm$0.38                                            & 88.67$\pm$1.87 & 14.85$\pm$0.09                                           \\
                         &                                                                                              & Ours   & 89.26$\pm$1.89 & 84.15$\pm$0.17                                            & 89.04$\pm$2.69 & 20.10$\pm$0.07                                            & 88.35$\pm$1.37 & 38.02$\pm$0.05                                            & 88.98$\pm$2.47 & 16.94$\pm$0.18                                           \\
                         & \multirow{2}{*}{$|\mathcal{N}(0,1)|$}                                                        & GSP    & 88.72$\pm$1.62 & 104.47$\pm$0.30                                           & 87.77$\pm$1.77 & 59.51$\pm$0.07                                            & 89.01$\pm$2.00 & 18.74$\pm$0.21                                            & 88.03$\pm$2.55 & 12.38$\pm$0.37                                           \\
                         &                                                                                              & Ours   & 88.93$\pm$2.40 & 80.86$\pm$0.04                                            & 88.60$\pm$1.52 & 82.61$\pm$0.13                                            & 88.03$\pm$2.41 & 32.93$\pm$0.05                                            & 88.61$\pm$1.85 & 44.16$\pm$0.11                                           \\ \midrule
\multirow{4}{*}{CIFAR10} & \multirow{2}{*}{$u(0,1)$}                                                                    & GSP    & 84.18$\pm$1.48 & 101.21$\pm$0.33                                           & 83.56$\pm$2.42 & 45.32$\pm$0.27                                            & 83.05$\pm$2.66 & 25.54$\pm$0.19                                            & 82.17$\pm$1.68 & 15.45$\pm$0.15                                           \\
                         &                                                                                              & Ours   & 83.16$\pm$2.09 & 119.51$\pm$0.24                                           & 84.04$\pm$1.86 & 47.96$\pm$0.15                                            & 83.40$\pm$2.68 & 35.66$\pm$0.32                                            & 82.45$\pm$1.74 & 16.69$\pm$0.08                                           \\
                         & \multirow{2}{*}{$|\mathcal{N}(0,1)|$}                                                        & GSP    & 84.26$\pm$2.57 & 127.80$\pm$0.17                                           & 82.73$\pm$2.13 & 57.28$\pm$0.13                                            & 82.88$\pm$1.63 & 29.81$\pm$0.17                                            & 82.78$\pm$2.33 & 18.12$\pm$0.21                                           \\
                         &                                                                                              & Ours   & 83.10$\pm$2.07 & 195.93$\pm$0.10                                           & 83.54$\pm$2.42 & 73.05$\pm$0.04                                            & 83.44$\pm$1.83 & 52.14$\pm$0.02                                            & 83.05$\pm$2.56 & 25.92$\pm$0.14                                           \\ \bottomrule
\end{tabular}}

\label{tb:param_s}

\end{table*}

\begin{figure*}[h]
\centering
\subfigure[Final Performance]{
\centering
\includegraphics[width=0.31\linewidth]{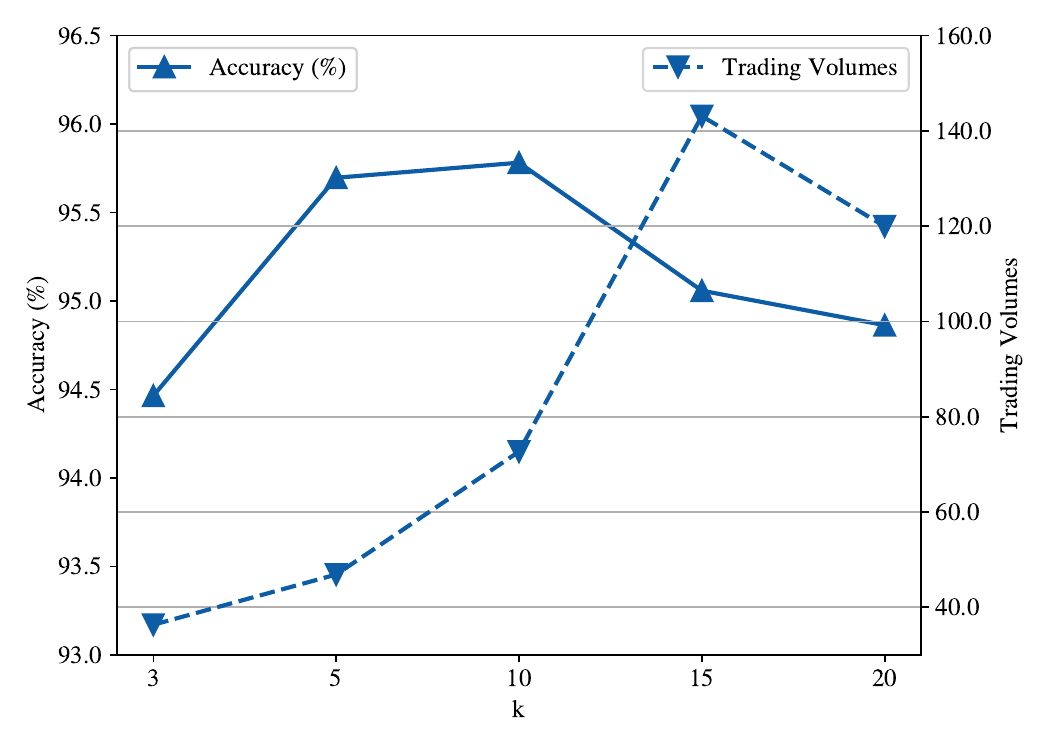}
}%
\subfigure[Accuracy]{
\centering
\includegraphics[width=0.28\linewidth]{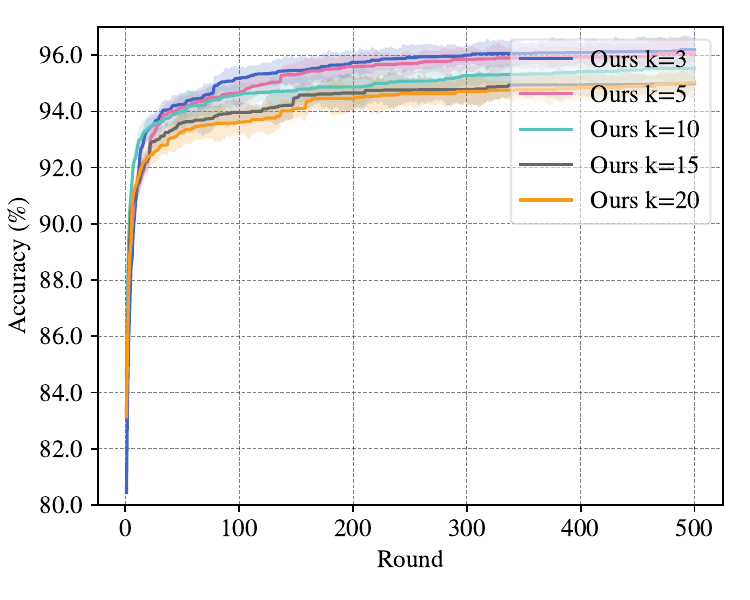}
}%
\subfigure[Trading Volumes]{
\centering
\includegraphics[width=0.285\linewidth]{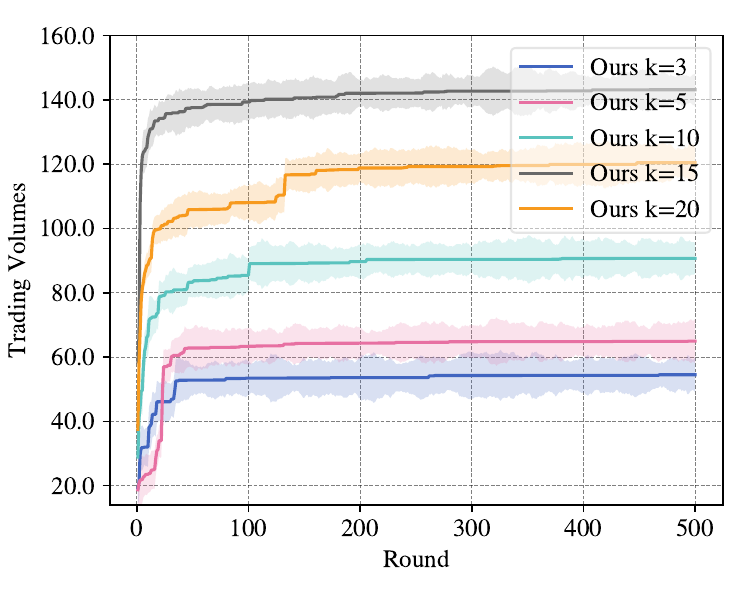}
}%
\centering
\vspace{-0.1in}
\caption{Effect of the number of model copies each seller can sell.}
\vspace{-0.1in}
\label{fig:param_k}
\end{figure*}

\textbf{Effect of Buyer Utility and the Ratio of Authorized Sellers.}
In Assumption \ref{ass:ir}, we assume each buyer client has a self-maintained utility rate $u_i$. We are interested in how the distribution of the value of $u_i$ affects the auction results. To evaluate this, we run experiments on two distributions: 1) $u(0,1)$, the uniform distribution between 0 and 1, and 2) $|\mathcal{N}(0,1)|$, the absolute value of normal distribution with the mean of 0 and std of 1. Moreover, recall that based on Theorem \ref{th:curse}, we propose to authorize a random subset of sellers for selling in each round. To evaluate how the ratio of active sellers influences the auction results, we run experiments on when $\{20\%,40\%,60\%,80\%\}$ of all sellers are authorized for selling.

In Table \ref{tb:param_s}, the experiment results are presented under the two groups of parameters on three datasets. The table shows the cumulative revenue and accuracy of the proposed method (Ours) and the comparison baseline (GSP). We make the following key observations. First, the results indicate that the proposed method is more effective in generating higher cumulative revenue and maintaining accuracy compared to the GSP baseline across different buyer utility distributions and authorized seller ratios. Second, the market reaches higher cumulative revenue at utility distribution $|\mathcal{N}(0,1)|$ compared with $u(0,1)$. This trend is more significant on MNIST and FMNIST datasets. On these datasets, the downstream task is simple and thus less performance gain is expected. In this case, since the auction setting aims to capture and exploit the variations in buyer preferences, a market with heterogeneous utility leads to improved auction outcomes. Third, though the performance of the market varies with different ratios of authorized sellers on both methods, the proposed method is relatively stable across different active seller ratios compared with GSP. 

\textbf{Effect of Number of Model Copies.} 
In Theory~\ref{th:copy}, we propose to allow only $k$ ($k<N$) model copies for sale, where $k$ should be pre-defined before the auction. To further evaluate how the value of $k$ affects the market, especially the auction results and training of RL-based solver, we run experiments with $k=\{3,5,10,15,20\}$.

The results on FMNIST dataset are plotted in Figure~\ref{fig:param_k}, (a) illustrates the final performance after convergebce, Figure (b) and Figure (c) illustrate the learning curve of the FL training between round 0-500. We can observe from Figure~\ref{fig:param_k} (a) that accuracy drops as $k$ goes larger. Similar trends are observed on other datasets. To avoid redundancy, we omit them here. The increase of $k$ indicates the increased probability of successful bidding, introducing more dynamics to the marketplace and making it more challenging for the policy net to find an optimal allocation strategy. This also explains the fluctuation of revenue as $k$ increases. A larger $k$ increases the total trading volumes (revenue first increases) but harms the allocation effectiveness (revenue then decreases), and when $k$ is too large, the latter's effect diminishes (revenue increases again). Another finding is that $k$ also affects the convergence speed of FL training. As shown in Figure~\ref{fig:param_k} (b) and Figure~\ref{fig:param_k} (c), a relative smaller $k$ ($k\leq$10) significantly speeds up convergence. However, it may not converge to a good performance. Therefore, practitioners should carefully decide on $k$, considering the trade-off between final accuracy and revenue, as well as convergence speed. A relatively medium number is a fair choice.


\end{document}